\documentclass[12pt]{article}
\usepackage{times}

\usepackage{amsthm,amsmath,amsfonts,color}
\usepackage{tcolorbox,amssymb}
\usepackage{hyperref}
\usepackage{cleveref}
\usepackage{natbib}

\usepackage{fullpage}
\newcommand{\vv}{\mathbf{v}}

\newcommand{\cH}{\mathcal{H}}
\newcommand{\ldim}[1]{\mathrm{Ldim}(#1)}
\newcommand{\ignore}[1]{}

\newcommand{\ord}{{\text{ord}}}
\newcommand{\N}{{\mathbb N}}
\DeclareMathOperator*{\Ex}{\mathbb{E}}
\newcommand{\twr}{\mathsf{twr}}
\newcommand{\R}{{\mathbb R}}
\renewcommand{\P}{{\mathcal P}}

\newcommand{\eps}{\epsilon}

\newcommand{\new}[1]{{ #1}}

\newtheorem{theorem}{Theorem}
\newtheorem{lemma}[theorem]{Lemma}
\newtheorem{claim}[theorem]{Claim}

\newtheorem{definition}[theorem]{Definition}
\newtheorem{corollary}[theorem]{Corollary}

\title{Private PAC learning implies finite Littlestone dimension}
\author{
Noga Alon\thanks{Department of Mathematics, Princeton University, Princeton, New Jersey, USA 
and Schools of Mathematics and Computer Science, Tel Aviv University, Tel Aviv, Israel.
Noga Alon's research is supported in part by ISF grant No. 281/17, GIF
grant No. G-1347-304.6/2016 and the Simons Foundation.}
\and Roi Livni\thanks{Department of Computer Science, Princeton University, Princeton, New Jersey, USA.
Research supported in part by the Eric and Wendy Schmidt foundation for strategic innovation.}
\and Maryanthe Malliaris\thanks{Department of Mathematics, University of Chicago, Chicago, Illinois, USA.
Research partially supported by NSF 1553653 and a Minerva research foundation membership at IAS.}
\and Shay Moran\thanks{Department of Computer Science, Princeton University, Princeton, New Jersey, USA.
Part of this research was done while the author was a member at the Institute for Advanced Study, Princeton
and was supported by the National Science Foundation under agreement No. CCF-1412958.}
}

\begin{document}
\maketitle

\abstract{
We show that every approximately differentially private learning algorithm (possibly improper) 
for a class $H$ with Littlestone dimension~$d$ requires $\Omega\bigl(\log^*(d)\bigr)$ examples. 
As a corollary it follows that the class of thresholds over $\N$ can not be learned in a private manner;
this resolves open questions due to \new{\citep{Bun15thresholds,Feldman15communication}}. 
We leave as an open question whether every class with a finite Littlestone dimension can be learned by an approximately differentially private algorithm.}

\section{Introduction}
Private learning concerns the design of learning algorithms
for problems in which the input sample contains sensitive data that needs to be protected.
Such problems arise in various contexts, including those involving 
social network data, financial records, medical records, etc. 
The notion of differential privacy~\citep{Dwork06calib,Dwork06ourdata},
which is a standard mathematical formalism of privacy,
enables a systematic study of algorithmic privacy in machine learning.
The question 
\begin{center}
``Which problems can be learned by a private learning algorithm?''
\end{center}
has attracted considerable attention~\citep{Rubinstein09large,Kasiv11learning,Chaudhuri11erm,Beimel13charac,Beimel14bounds,Chaudhuri14margin,Balcan15active,Bun15thresholds,Beimel15unlabeled,Feldman15communication,Wang16learning,Cummings16robust,Beimel16sanit,Bun16direct,Bassily16stability,Ligett17accuracy,Bassily18model,Feldman18prediction}.

{\it Learning thresholds} is one of the most basic problems in machine learning.
This problem consists of an unknown threshold function $c:\R \to\{\pm 1\}$, 
an unknown distribution~$D$ over $\R$, 
and the goal is to output an hypothesis $h:\R\to \{\pm 1\}$ that is close to $c$,
given access to a limited number of input examples $(x_1,c(x_1)),\ldots,(x_m,c(x_m))$,
where the~$x_i$'s are drawn independently from~$D$.

The importance of thresholds stems from that it appears as a subclass of many other well-studied classes.
For example, it is the one dimensional version of the class of {\it Euclidean Half-Spaces}
which underlies popular learning algorithms such as kernel machines and neural networks (see e.g.\ \cite{Shalev14book}).

Standard PAC learning of thresholds without privacy constraints is known to be easy and can be done using a constant number of examples. 
In contrast, whether thresholds can be learned privately turned out to be more challenging to decide, 
and there has been an extensive amount of work that addressed this task:
the work by \cite{Kasiv11learning} implies a {\it pure differentially private} algorithm (see the next section for formal definitions) 
that learns thresholds over a finite $X\subseteq \R$ of size $n$ with $O(\log n)$ examples.
\cite{Feldman15communication} showed a matching lower bound for any pure differentially private algorithm.
\cite{Beimel16sanit} showed that by relaxing the privacy constraint to {\it approximate differential privacy},
one can significantly improve the upper bound to some $2^{O(\log^*(n))}$.
\cite{Bun15thresholds} further improved the upper bound from \citep{Beimel16sanit} by polynomial factors and 
gave a lower bound of $\Omega(\log^*(n))$ that applies for any proper learning algorithm.
They also explicitly asked whether the dependence on $n$ can be removed in the improper case. 
\new{\cite{Feldman15communication} asked more generally whether any class can be learned
privately with sample complexity depending only on its VC dimension (ignoring standard dependencies on the privacy and accuracy parameters).}
Our main result (\Cref{thm:main}) answers these questions by showing that a similar lower bound applies 
for any (possibly improper) learning algorithm.


Despite the impossibility of privately learning thresholds,
there are other natural learning problems that can be learned privately,
In fact, even for the class of Half-spaces, 
private learning is possible if the target half-space
satisfies a large {\it margin\footnote{The margin is a geometric measurement 
for the distance between the separating hyperplane
and typical points that are drawn from the target distribution.}}
assumption \citep{Blum05practical,Chaudhuri11erm}.

Therefore, it will be interesting to find a natural invariant that characterizes
which classes can be learned privately
(like the way the VC dimension characterizes PAC learning~\citep{Blumer89learnability,Vapnik71uniform}).
Such parameters exist in the case of pure differentially private learning;
these include the {\it one-way communication complexity} characterization by~\cite{Feldman15communication}
and the {\it representation dimension} by~\cite{Beimel13charac}. 
However, no such parameter is known for approximate differentially private learning.
We next suggest a candidate invariant that rises naturally from this work.

\subsection{Littlestone dimension vs. approximate private learning}
The Littlestone dimension~\citep{Littlestone87online} is a combinatorial parameter 
that characterizes learnability of binary-labelled classes within {\it Online Learning}
both in the realizable case~\citep{Littlestone87online} and in the agnostic case~\citep{Bendavid09agnostic}.

It turns out that there is an intimate relationship between thresholds and the Littlestone dimension:
a class $H$ has a finite Littlestone dimension
if and only if it does not embed thresholds as a subclass 
(for a formal statement, see \Cref{{thm:shelah}});
this follows from a seminal result in model theory by \cite{Shelah78classification}. 
As explained below, Shelah's theorem is usually stated in terms of orders and ranks. 
\cite{chase2018model} noticed\footnote{Interestingly, though the Littlestone dimension is a basic parameter in Machine Learning (ML), 
this result has not appeared in the ML literature.}  that the Littlestone dimension is the same as the model theoretic rank. 
Meanwhile, order translates naturally to thresholds. 
To make \Cref{thm:shelah} more accessible for readers with less background in model theory, 
we provide a combinatorial proof in the appendix.


While it still remains open whether finite Littlestone dimension
is indeed equivalent to private learnability,
our main result (\Cref{thm:main}) combined with the above connection 
between Littlestone dimension and thresholds (\Cref{thm:shelah})
imply an implication in one direction:
At least $\Omega(\log^* d)$ examples are required for privately learning any class with Littlestone dimension $d$
(see \Cref{thm:ADPimpliesLD}).

It is worth noting that \cite{Feldman15communication} studied the Littlestone dimension
in the context of pure differentially private learning:
(i) they showed that $\Omega(d)$ examples are required for learning a class with Littlestone dimension $d$
in a pure differentially private manner,
(ii) they exhibited classes with Littlestone dimension $2$ that can not be learned
by pure differentially private algorithms,
and (iii) they showed that these classes can be learned by approximate differential private algorithms.

\paragraph{Organization}
The rest of this manuscript is organized as follows: 
\Cref{sec:mainresult} presents the two main results,
\Cref{sec:pre} contains definitions and technical background
from machine learning and differential privacy, 
and \Cref{sec:thresholds} and \Cref{sec:lsdim} contain the proofs.

\section{Main Results}\label{sec:mainresult}
We next state the two main results of this paper. 
The statements use technical terms from differential privacy
and machine learning whose definitions appear in \Cref{sec:pre}.

We begin by the following statement that resolves an open problem \new{in \cite{Feldman15communication}
and \cite{Bun15thresholds}}: 
\begin{theorem}[Thresholds are not privately learnable]\label{thm:main}
\new{Let $X\subseteq \R$ of size $\lvert X\rvert = n$ and let $\cal A$
be a $(\frac{1}{16},\frac{1}{16})$-accurate learning algorithm for the class of thresholds over $X$ 
with sample complexity $m$ which satisfies
$(\eps,\delta)$-differential privacy with $\eps=0.1$ and $\delta = O(\frac{1}{m^2\log m})$.
Then,  
\[m\geq \Omega(\log^*n).\]
In particular, the class of thresholds over an infinite $X$ can not be learned privately.}
\end{theorem}

\Cref{thm:main} and \Cref{thm:shelah} (which is stated in the next section) imply
that any privately learnable class has a finite Littlestone dimension.
As elaborated in the introduction, this extends a result by~\cite{Feldman15communication}.

\begin{corollary}[Private learning implies finite Littlestone dimension]\label{thm:ADPimpliesLD}
\new{Let $H$ be an hypothesis class with Littlestone dimension $d\in\N\cup\{\infty\}$ and let $\cal A$
be a $(\frac{1}{16},\frac{1}{16})$-accurate learning algorithm for $H$ 
with sample complexity $m$ which satisfies
$(\eps,\delta)$-differential private with $\eps=0.1$ and $\delta = O(\frac{1}{m^2\log m})$.
Then,  
\[m\geq \Omega(\log^*d).\]
In particular any class that is privately learnable has a finite Littlestone dimension.}
\end{corollary}

\section{Preliminaries}\label{sec:pre}

\subsection{PAC learning}
We use standard notation from statistical learning, see e.g.\ \citep{Shalev14book}.
Let $X$ be a set and let $Y=\{\pm 1\}$.
An {\it hypothesis} is an $X\to Y$ function.
An {\it example} is a pair in $X\times Y$. 
A {\it sample} $S$ is a finite sequence of examples.
The {\it loss of $h$ with respect to $S$} is defined by
\[L_S(h) = \frac{1}{\lvert S\rvert}\sum_{(x_i,y_i)\in S}1[h(x_i)\neq y_i].\]
The {\it loss of $h$ with respect to a distribution} $D$ over $X\times Y$ is defined by
\[L_D(h) = \Pr_{(x,y)\sim D} [h(x)\neq y].\]
Let $\cH\subseteq Y^X$ be an {\it hypothesis class}.
$S$ is said to be {\it realizable by $\cH$} if there is $h\in H$ such that $L_S(h)=0$ .
$D$  is said to be {\it realizable by $\cH$} if there is $h\in H$ such that $L_D(h)=0$.
A {\it learning algorithm} $A$ is a (possibly randomized) mapping taking input samples to output hypotheses.
We denote by $A(S)$ the distribution over hypotheses induced by the algorithm when the input sample is $S$.
We say that $A$ {\it learns\footnote{We focus on the realizable case.} a class $\cH$}
with $\alpha$-{\it error},  $(1-\beta)$-{\it confidence}, and {\it sample-complexity} $m$
if for every realizable distribution $D$:
\[\Pr_{S\sim D^m,~h\sim A(S)}[L_D(h) > \alpha] \leq \beta,\]
For brevity if $A$ is a learning algorithm with $\alpha$-error and $(1-\beta)$-confidence we will say that~$A$ is an \emph{$(\alpha,\beta)$-accurate learner}.

\paragraph{Littlestone Dimension}
The Littlestone dimension is a combinatorial 
parameter that characterizes regret bounds in Online Learning \citep{Littlestone87online,Bendavid09agnostic}. 
The definition of this parameter uses the notion of {\it mistake-trees}:
these are binary decision trees whose internal nodes are labelled by elements of $X$.
Any root-to-leaf path in a mistake tree can be described as a sequence of examples 
$(x_1,y_1),...,(x_d,y_d)$, where $x_i$ is the label of the $i$'th 
internal node in the path, and $y_i=+1$ if the $(i+1)$'th node  
in the path is the right child of the $i$'th node, and otherwise $y_i = -1$.
We say that a tree $T$ is {\it shattered }by $\cH$ if for any root-to-leaf path
$(x_1,y_1),...,(x_d,y_d)$ in $T$ there is $h\in \cH$ such that $h(x_i)=y_i$, for all $i\leq d$.
The Littlestone dimension of $\cH$, denoted by $\ldim{\cH}$, is the depth of largest
complete tree that is shattered by~$\cH$.

%


Recently, \cite{chase2018model} noticed that the Littlestone dimension coincides with a model-theoretic 
measure of complexity, Shelah's $2$-rank.

A classical theorem of Shelah connects bounds on 2-rank (Littlestone dimension)  
to bounds on the so-called order property in model theory. The order property corresponds naturally to the concept of {\it thresholds}. 
Let $\cH\subseteq \{\pm 1\}^X$ be an hypothesis class. We say that $\cH$ {\it contains $k$ thresholds}
if there are $x_1,\ldots,x_k\in X$ and $h_1,\ldots,h_k\in \cH$ such that $h_i(x_j) = 1$
if and only if $i\leq j$ for all~$i,j\leq k$.
   
Shelah's result (part of the so-called Unstable Formula Theorem\footnote{\cite{Shelah78classification} provides a qualitative statement, 
a quantitative one that is more similar to \Cref{thm:shelah} can be found at~\cite{Hodges97book}})
\citep{Shelah78classification, Hodges97book}, which we use in the following translated form, provides a simple and elegant connection between 
Littlestone dimension and thresholds. 

\begin{theorem}\label{thm:shelah}(Littlestone dimension and thresholds \citep{Shelah78classification, Hodges97book})\ \\
Let $\cH$ be an hypothesis class, then:
\begin{enumerate}
\item If the $\ldim{\cH}\geq d$ then $\cH$ contains $\lfloor \log d\rfloor$ thresholds
\item If $\cH$ contains $d$ thresholds then its  $\ldim{\cH}\geq \lfloor \log d\rfloor$.
\end{enumerate}
\end{theorem}
For completeness, we provide a combinatorial proof of \Cref{thm:shelah} in \Cref{sec:shelah}.

In the context of model theory, \Cref{thm:shelah} is used to establish an equivalence between 
finite  
Littlestone dimension and 
\emph{stable theories}.  
It is interesting to note that an analogous 
connection between theories that are called  \emph{NIP theories} and VC dimension has also been previously observed and was pointed out by \cite{laskowski1992vapnik}; this in turn led to results in Learning theory: 
in particular within the context of compression schemes  \citep{livni2013honest} 
but also some of the first polynomial bounds for the VC dimension for sigmoidal neural networks \citep{karpinski1997polynomial}.

\subsection{Privacy}
We use standard notation from differential privacy.
For more background see e.g.\ the surveys~\citep{Dwork14survey,Vadhan17survey}. 
For $s,t\in \R$ let $a=_{\eps,\delta} b$ denote the statement
\[a\leq e^{\eps}b + \delta ~\text{  and  }~   b\leq e^\eps a + \delta.\]
We say that two distributions $p,q$ are {\it $(\eps,\delta)$-indistinguishable} if 
$p(E) =_{\eps,\delta} q(E)$ for every event~$E$.
Note that when~$\eps=0$ this specializes to the total variation metric.
\begin{definition}[Private Learning Algorithm]\label{def:private}
A randomized learning algorithm 
\[A: (X\times \{\pm 1\})^m \to \{\pm 1\}^X\] 
is $(\eps,\delta)$-differentially private 
if for every two samples $S,S'\in (X\times \{\pm 1\})^m$ that disagree on a single example,  
the output distributions $A(S)$ and $A(S')$ are $(\eps,\delta)$-indistinguishable.
\end{definition}
The parameters $\eps,\delta$ are usually treated as follows: $\eps$ is a small constant (say $0.1$),  and $\delta$ is negligible, $\delta = m^{-\omega(1)}$, where $m$ is the input sample size. The case of $\delta=0$ is also referred to as {\it pure differential privacy}.
A common interpretation of a negligible~$\delta > 0$ is that there is a tiny chance of a catastrophic event (in which perhaps all the input data is leaked) but otherwise the algorithm satisfies pure differential privacy.
Thus, a class $\cH$ is privately learnable if it is PAC learnable by an algorithm $A$
that is $(\eps(m),\delta(m))$-differentially private with $\eps(m) \leq o(1)$, and $\delta(m) \leq m^{-\omega(1)} $.

We will use the following corollary of the {\it Basic Composition Theorem} from differential privacy (see, e.g.\ Theorem 3.16 in \citep{Dwork14survey}).
\begin{lemma}\label{lem:prod}\citep{Dwork06ourdata,Dwork09robust}
If $p,q$ are $(\eps,\delta)$-indistinguishable then for all $k\in\mathbb{N}$,
$p^k$ and $q^k$ are $(k\eps,k\delta)$-indistinguishable, where $p^k,q^k$ are the k-fold products of~$p,q$ (i.e.\ corresponding to $k$ independent samples).
\end{lemma}
For completeness, a proof of this statement appears in \Cref{app:prod}.

\paragraph{Private Empirical  Learners}
It will be convenient to consider the following task of minimizing the empirical loss.
%

\begin{definition}[Empirical Learner]
Algorithm $A$ is $(\alpha,\beta)$-accurate empirical learner for a hypothesis class $\cH$ with sample complexity $m$ if for every $h\in \cH$ and for every sample $S=((x_1,h(x_1),\ldots, (x_m,h(x_m)))\in \left(X\times \{0,1\}\right)^m$ the algorithm $A$ outputs a function~$f$ satisfying
\[\Pr_{f\sim A(S)}\bigl(L_{S}(f) \le \alpha\bigr)\ge 1-\beta\]
\end{definition}

This task is simpler to handle than PAC learning, which is a distributional loss minimization task.
Replacing PAC learning by this task does not lose generality;
this is implied by the following result by \cite{Bun15thresholds}.

\begin{lemma}\label{lem:bun}[\cite{Bun15thresholds}, Lemma 5.9]
Suppose \new{$\eps<1$} and $A$ is an $(\epsilon,\delta)$-differentially private $(\alpha,\beta)$--accurate learning algorithm for a hypothesis class $\cH$ with sample complexity~$m$. 
Then there exists an $(\epsilon,\delta)$--differentially private $(\alpha,\beta)$--accurate empirical learner for $\cH$ with sample complexity $9m$.
\end{lemma}

\subsection{Additional notations}
A sample $S$ of an even length is called \emph{balanced} if  half of its labels are $+1$'s and half are $-1$'s.

For a sample $S$, let $S_X$ denote the underlying
set of unlabeled examples: $S_X = \bigl\{x \vert (\exists y): (x,y)\in S\bigr\}$.
Let~$A$ be a randomized learning algorithm.
It will be convenient to associate with~$A$ and $S$ the function $A_S:X\to[0,1]$
defined by
\[ A_S(x) = \Pr_{h\sim A(S)}\bigl[h(x)=1\bigr].
\]
Intuitively, this function represents the average hypothesis outputted by $A$
when the input sample is $S$.

For the next definitions assume that the domain $X$ is linearly ordered.
Let $S=((x_i,y_i))_{i=1}^{m}$ be a sample. 
We say that $S$ is {\it increasing} if $x_1< x_2< \ldots< x_m$.
For $x\in X$ define $\ord_S(x)$ by $\lvert\{ i \vert x_i \leq x\}\rvert$.
Note that the set of points $x\in X$ with the same $\ord_S(x)$ form an interval
whose endpoints are two consecutive examples in $S$ (consecutive with respect to the order on $X$, 
i.e.\ there is no example $x_i$ between them).

The {\it tower function} $\twr_k(x)$ is defined by the recursion
\[
\twr^{(i)} x = 
\begin{cases}
x &i = 1,\\
2^{\twr{(i-1)}(x)} &i> 1. 
\end{cases}
\]
The iterated logarithm, $\log^{(k)}(x)$ is defined by the recursion
\[
\log^{(i)} x = 
\begin{cases}
\log x &i = 0,\\
1 + \log^{(i-1)}\log x &i> 0. 
\end{cases}
\]
The function $\log^*x$ equals the number of times the iterated logarithm must be applied before the result is less than or equal to $1$. 
It is defined by the recursion 
\[
\log^* x = 
\begin{cases}
0 &x\leq 1,\\
1 + \log^*\log x &x>1. 
\end{cases}
\]

\section{A lower bound for privately learning thresholds}\label{sec:thresholds}

In this section we prove \Cref{thm:main}.

\subsection{Proof overview}
We begin by considering an arbitrary differentially private algorithm $A$ that learns the class of thresholds over an ordered domain $X$ of size $n$.
Our goal is to show a lower bound of~$\Omega(\log^* n)$ on the sample complexity of $A$.
A central challenge in the proof follows because $A$ may be improper and output arbitrary hypotheses
(this is in contrast with proving impossibility results for proper algorithms where the structure of the learned class can be exploited).

The proof consists of two parts:
(i) the first part handles the above challenge by showing that for any algorithm
(in fact, for any mapping that takes input samples to output hypotheses)
there is a large subset of the domain that is {\it homogeneous with respect to the algorithm}.
This notion of homogeneity places useful restrictions on the algorithm
when restricting it to the homogeneous set. 
(ii) The second part of the argument utilizes such a large homogeneous set $X'\subseteq X$
to derive a lower bound on the sample complexity of the algorithm in terms of~$\lvert X' \rvert$.

\new{We note that the Ramsey argument in the first part is quite general:
it does not use the definition of differential privacy and could perhaps be useful
in other sample complexity lower bounds. 
Also, a similar argument was used by~\cite{bun16thesis}
in a weaker lower bound for privately learning thresholds in the proper case.
However, the second and more technical part of the proof is tailored specifically to the definition of differential privacy.}
We next outline each of these two parts.

\paragraph{Reduction to an algorithm over an homogeneous set}

As discussed above, the first step in the proof is about
identifying a large homogeneous subset of the input domain $X$ on which we can control the output of $A$:
a subset $X'\subseteq X$ is called {\it homogeneous with respect to~$A$} 
if there is a list of numbers $p_0,p_1,\ldots,p_m$ such that for every increasing balanced
sample~$S$ of points from $X'$ and for every $x'$ from $X'$ with $\ord_S(x') = i$:
\[\lvert A_S(x') - p_i\rvert \leq \gamma,\]
where $\gamma$ is sufficiently small.
For simplicity, in this proof-overview we will assume that~$\gamma=0$ 
(in the formal proof $\gamma$ is some $O(1/m)$ - see \Cref{def:homog}).
So, for example, if $A$ is deterministic then $h=A(S)$ is constant
over each of the intervals defined by consecutive examples from $S$.
See \Cref{fig:homogenic} for an illustration of homogeneity.

The derivation of a large homogeneous set follows by a standard application of Ramsey Theorem for hyper-graphs
using an appropriate coloring (\Cref{lem:ramsey}).


%

\begin{figure}
\includegraphics[scale=0.5]{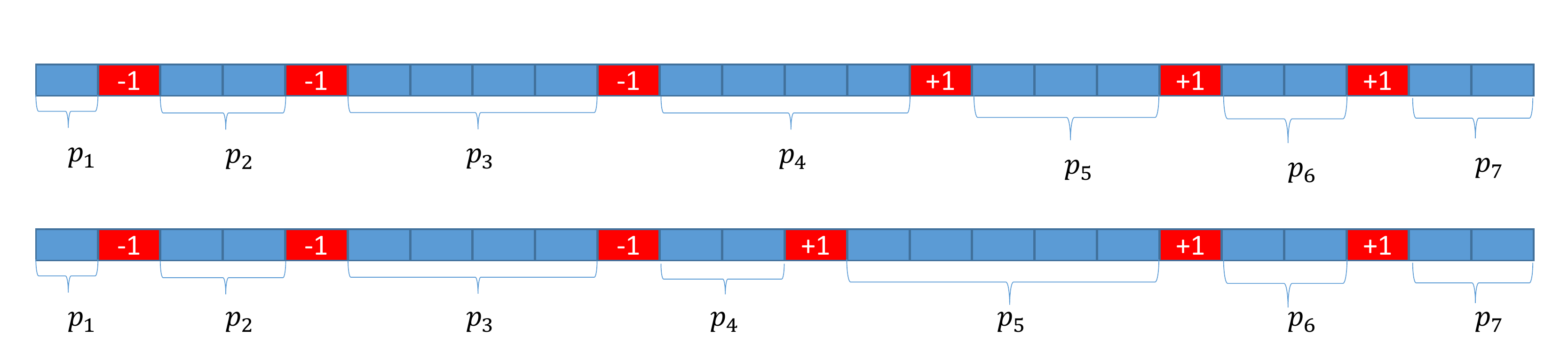}
\caption{\small{Depiction of two possible outputs of an algorithm over an homogeneous set, given two input samples from the set (marked in red). 
The number $p_i$ denote, for a given point $x$, the probability that $h(x)=1$, where $h\sim A(S)$ is the hypothesis $h$ outputted by the algorithm on input sample $S$. 
These probabilities depends (up to a small additive error) only on the interval that $x$ belongs to. 
In the figure above we changed in the input the fourth example -- this only affects the interval 
and not the values of the $p_i$'s (again, up to a small additive error).}}\label{fig:homogenic}
\end{figure}

\paragraph{Lower bound for an algorithm defined on large homogeneous sets}
We next assume that $X'=\{1,\ldots,k\}$ is a large homogeneous set with respect to $A$ (with $\gamma=0$).
We will obtain a lower bound on the sample complexity of $A$, denoted by $m$, 
by constructing a family $\P$ of distributions such that:
(i) on the one hand $\lvert \P \rvert \leq 2^{\tilde O(m^2)}$, 
and (ii) on the other hand $\lvert \P \rvert\geq \Omega( k)$.
Combining these inequalities yields a lower bound on $m$ and concludes the proof.

The construction of $\P$ proceeds as follows and is depicted in \Cref{fig:AtoP}:
let $S$ be an increasing balanced sample of points from $X'$.
Using the fact that $A$ learns thresholds it is shown
that for some $i_1<i_2$ we have that $p_{i_1}\leq 1/3$ and $p_{i_2} \geq 2/3$.
Thus, by a simple averaging argument
there is some $i_1\leq i \leq i_2$ such that $p_{i} - p_{i-1} \geq \Omega(1/m)$.

The last step in the construction is done by picking an increasing sample $S$ 
such that the interval $(x_{i-1}, x_{i+1})$ has size $n=\Omega(k)$.
For $x\in (x_{i-1}, x_{i+1})$, let $S_x$ denote the sample obtained by replacing $x_i$ with $x$ in $S$.
Each output distribution $A(S_x)$ can be seen as a distribution over the cube $\{\pm 1\}^n$
(by restricting the output hypothesis to the interval $(x_{i-1}, x_{i+1})$, which is of size $n$).
This is the family of distributions $\P=\{P_j : j\leq n\}$.
Since $A$ is private, and by choice of the interval $(x_i,x_{i+1})$ we obtain that $\P$ has the following two properties:
\begin{itemize}
\item $P_{j'}, P_{j''}$ are $(\eps,\delta)$-indistinguishable for all $j',j''$, and
\item  Put $r=\frac{p_{i-1} + p_{i}}{2}$, then for all $P_j$
\[(\forall x\leq n): \Pr_{v\sim P_j}\bigl[v(x)=1\bigr]
=
\begin{cases}
r-\Omega(1/m)   &x<j,\\
r+\Omega(1/m) &x>j.
\end{cases}
\]
\end{itemize}
It remains to show that $\Omega(k) \leq \lvert \P\rvert \leq 2^{\tilde O(m^2)}$.
The lower bound follows directly from the definition of $\P$. 
The upper bound requires a more subtle argument:
it exploits the assumption that $\delta$ is small and \Cref{lem:prod}
via a binary-search argument and concentration bounds.
This argument appears in \Cref{lem:binary}, whose proof is self-contained.

\begin{figure}[h]
\includegraphics[scale=0.45]{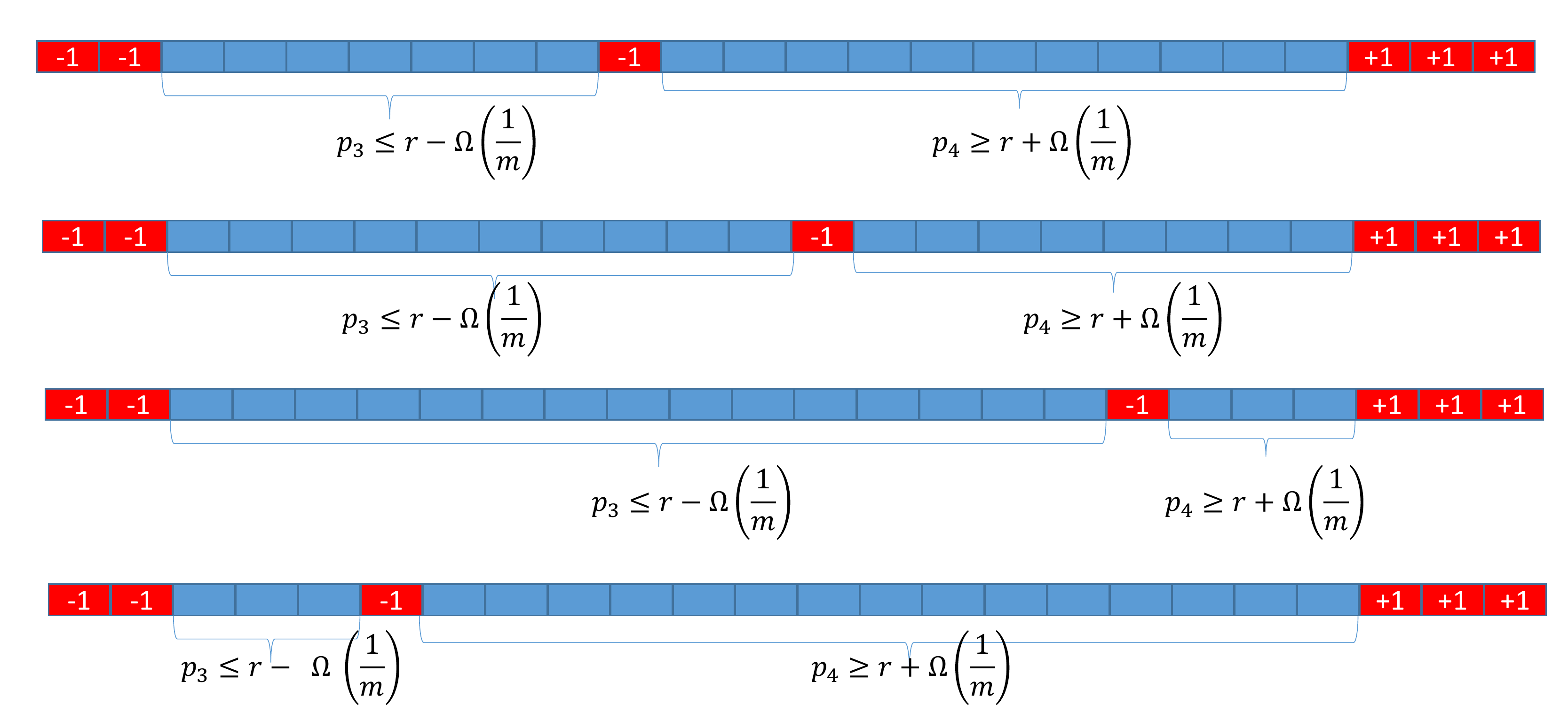}
\caption{\small{An illustration of the definition of the family $P$. Given an homogeneous set and two consecutive intervals where there is a gap of at least $\Omega(1/m)$ between $p_i$ and $p_{i-1}$ (here $i=4$). The distributions in $P$ correspond to the different positions of the $i$'th example, which separates between the $(i-1)$'th and the $i$'th intervals.}}\label{fig:AtoP}
\end{figure}


%

\subsection{Proof of \Cref{thm:main}}

The proof uses the following definition of homogeneous sets.
Recall the definitions of balanced sample and of an increasing sample.
In particular that a sample $S=((x_1,y_1),\ldots,(x_m,y_m))$ of an even size is realizable (by thresholds),
balanced, and increasing if and only if $x_1<x_2<\ldots<x_m$ and the first half of the $y_i$'s are $-1$
and the second half are $+1$.
\begin{definition}[$m$-homogeneous set]\label{def:homog}
A set $X'\subseteq X$ is {\it $m$-homogeneous} with respect to a learning algorithm $A$
if there are numbers $p_i\in [0,1]$, for $0\leq i\leq m$ 
such that for every increasing balanced realizable sample $S\in \bigl(X' \times \{\pm 1\}\bigr)^m$
and for every $x\in X'\setminus S_X$:
\[\bigl\lvert A_S(x) - p_i\bigr\rvert \leq \frac{1}{10^2 m},\]
where $i = \ord_S(x)$. 
The list $(p_i)_{i=0}^m$ is called the probabilities-list of $X'$ with respect to $A$.
\end{definition}

\begin{proof}[Proof of \Cref{thm:main}]
Let $A$ be a $(1/16,1/16)$-accurate learning algorithm that learns the class of thresholds over $X$ 
with $m$ examples and is $(\eps,\delta)$-differential private with $\eps=0.1,\delta = \frac{1}{10^3m^2\log m}$.
By \Cref{lem:bun} we may assume without loss of generality that $A$ is an empirical learner
with the same privacy and accuracy parameters and sample size that is at most 9 times larger.

\Cref{thm:main} follows from the following two lemmas:
\begin{lemma}[Every algorithm has large homogeneous sets]\label{lem:finiteramsey}\label{lem:ramsey}
Let $A$ be a (possibly randomized) algorithm that is defined over input samples of size $m$
over a domain $X\subseteq R$ with $\lvert X\rvert = n$.
Then, there is a set  $X'\subseteq X$ that is $m$-homogeneous with respect to~$A$ of size
\[ \lvert X'\rvert \geq \frac{\log^{(m)}(n)}{2^{O(m\log m)}}.\]
\end{lemma}

\Cref{lem:ramsey} allows us to focus on a large homogeneous set with respect to $A$.
The next Lemma implies a lower bound in terms of the size of a homogeneous set.
For simplicity and without loss of generality assume that the homogeneous set
is $\{1,\ldots,k\}$.

\begin{lemma}[Large homogeneous sets imply lower bounds for private learning]\label{lem:lbhomog}
Let $A$ be an $(0.1,\delta)$-differentially private algorithm with sample complexity $m$ and $\delta \leq \frac{1}{10^3m^2\log m}$.
Let~$X=\{1,\ldots, k\}$ be $m$-homogeneous with respect to~$A$.
Then, if $A$ empirically learns the class of thresholds over $X$ with $(1/16,1/16)$-accuracy, then
\[k \leq 2^{O(m^2\log^2m)}\] 
(i.e.~$m \geq \Omega\Bigl(\frac{\sqrt{\log k}}{\log\log k}\Bigr)$).
\end{lemma}

We prove \Cref{lem:ramsey} and \Cref{lem:lbhomog} in the following two subsections.

With these lemmas in hand, \Cref{thm:main} follows by a short calculation:
indeed, \Cref{lem:ramsey} implies the existence of an homogeneous set $X'$ with respect to $A$
of size $k\geq {\log^{(m)}(n)}/{2^{O(m\log m)}}$. We then restrict $A$ to input samples from the set $X'$,  and by relabeling the elements of $X'$ assume that $X'=\{1,\ldots,k\}$ .
\Cref{lem:lbhomog} then implies that $k = 2^{O(m^2\log^2m)}$.
Together we obtain that 
\[ \log^{(m)}(n) \leq 2^{c\cdot m^2\log m} \]
for some constant $c > 0$.
Applying the iterated logarithm $t=\log^*(2^{c\cdot m^2\log m}) = \log^{*}(m)+O(1)$ 
times on the inequality yields that
\[\log^{(m+t)}(n)=\log^{(m + \log^*(m) + O(1))}(n) \leq 1,\]
and therefore $\log^*(n) \leq \log^*(m) + m +O(1)$,
which implies that $m \geq \Omega(\log^* n)$ as required.

\end{proof}
\subsection{Proof of \Cref{lem:ramsey}}

We next prove that every learning algorithm has a large homogeneous set. 
We will use the following quantitative version of Ramsey Theorem due to \cite{erdos52combinatorial} 
(see also the book \citep{graham90ramsey}, or Theorem 10.1 in the survey by~\cite{mubayi17survey}):
\begin{theorem}\label{thm:ramsey}\citep{erdos52combinatorial}
Let $s>t\geq 2$ and $q$ be integers, and let 
\[N\geq \twr_t(3sq\log q).\] 
Then for every coloring of the subsets of size $t$ of a universe of size $N$ using $q$ colors
there is a homogeneous subset\footnote{A subset of the universe is homogeneous if all of its $t$-subsets have the same color.} of size $s$.
\end{theorem}

\begin{proof}[Proof of \Cref{lem:ramsey}]
Define a coloring on the $(m+1)$-subsets of $X$ as follows.
Let $D=\{x_1<x_2<\ldots<x_{m+1}\}$ be an $(m+1)$-subset of~$X$.
For each $i\leq m+1$ let $D^{-i} = D\setminus\{x_i\}$, 
and let $S^{-i}$ denote the balanced increasing sample on $D^{-i}$. 
Set $p_{i}$ to be the fraction of the form~$\frac{t}{10^2m}$
that is closest to $A_{S^{-i}}(x_{i})$  (in case of ties pick the smallest such fraction).
The coloring assigned to $A$ is the list $(p_1,p_2,\ldots,p_{m+1})$.

Thus, the total number of colors is~$(10^2m+1)^{(m+1)}$.
By applying \Cref{thm:ramsey} with $t:=m+1, q:=(10^2m+1)^{(m+1)}$, and $N:=n$ there is a set $X'\subseteq X$ of size 
\[\lvert X'\rvert \geq \frac{\log^{(m)}(n)}{{3(10^2m+1)^{m+1} (m+1)\log(10^2m+1)}} =  \frac{\log^{(m)}(N)}{2^{O(m\log m)}}\]
such that all $m+1$-subsets of $X'$ have the same color.
One can verify that $X'$ is indeed $m$-homogeneous with respect to $A$.
\end{proof}

\subsection{Proof of \Cref{lem:lbhomog}}

The lower bound is proven by using the algorithm $A$
to construct a family of distributions $\P$ with certain properties,
and use these properties to derive that $\Omega(k) \leq \P \leq 2^{O(m^2\log^2 m)}$,
which implies the desired lower bound.
%

\begin{lemma}\label{lem:AtoP}
%
Let $A,X',m,k$ as in \Cref{lem:lbhomog}, and set $n=k-m$. 
Then there exists a family $\P=\{P_i : i\le n\}$ of distributions over $\{\pm 1\}^n$ with the following properties:
\begin{enumerate}
\item Every $P_i,P_j\in \P$ are $(0.1,\delta)$-indistinguishable.
\item There exists $r\in [0,1]$  such that for all $i,j\le n$: 
\[\Pr_{v\sim P_i}\bigl[v(j)=1\bigr]
=
\begin{cases}
\leq r-\frac{1}{10m} &j < i,\\
\geq r+\frac{1}{10m} &j>i.
\end{cases}\]
\end{enumerate}
\end{lemma}

\begin{lemma}\label{lem:binary}
Let $\P,n,m,r$ as in \Cref{lem:AtoP}.
Then $n \leq 2^{10^3 m^2\log^2m}$.
\end{lemma}

By the above lemmas, $k-m  = \lvert \P\rvert \leq 2^{10^3m^2\log^2 m}$,
which implies that $k=2^{O(m^2\log^2  m)}$ as required.
Thus, it remains to prove these lemmas, which we do next.

\subsubsection{Proof of \Cref{lem:AtoP}}

For the proof of \cref{lem:AtoP} we will need the following claim:

\begin{claim}\label{lem:reduction} 
Let $(p_i)_{i=0}^m$ denote the probabilities-list of $X'$ with respect to $A$.
Then for some $0 < i \leq m$:
\[p_{i} - p_{i-1} \geq \frac{1}{4m}\] 
\end{claim}
%
\begin{proof}
The proof of this claim uses the assumption that $A$ empirically learns thresholds. 
Let $S$ be a balanced increasing realizable sample such that $S_X=\{x_1<\ldots <x_m\}\subseteq X'$ are evenly spaced points on $K$
(so, $S=(x_i,y_i)_{i=1}^m$, where $y_i = -1$ for $i\leq m/2$ and $y_i=+1$ for $i> m_2$).

$A$ is an $(\alpha=1/16,\beta=1/16)$-empirical learner and therefore its expected empirical loss on~$S$ 
is at most $(1-\beta)\cdot\alpha + \beta\cdot 1\leq \alpha + \beta = 1/8$, and so:
\begin{align*}
\frac{7}{8}& \le \mathop\mathbb{E}_{h\sim A(S)} (1-L_S(h))\\&= 
\frac{1}{m}\sum_{i=1}^{m/2} \left[1-A_S(x_i) \right]+\frac{1}{m}\sum_{i=m/2+1}^m \left[ A_S(x_i) \right]. \tag{since $S$ is balanced}
\end{align*}
This implies that there is $m/2\le m_1\le m$ such that $A_S(x_{m_1})\ge 3/4$.
Next, by privacy if we consider $S'$ the sample where we replace $x_{m_1}$ by $x_{m_1}+1$ (with the same label), we have that \[A_{S'}(x_{m_1}) \ge \Bigl(\frac{3}{4} -\delta\Bigr) e^{-0.1} \ge \frac{2}{3}.\]
Note that $\ord_{S'}(x_{m_1}) = m_1-1$, hence by homogeneity: $p_{m_1-1}\ge\frac{2}{3}- \frac{1}{10^2m}$.
Similarly we can show that for some $1 \le m_2\le \frac{m}{2}$ we have $p_{m_2-1}\le\frac{1}{3} + \frac{1}{10^2m}$.
This implies that for some $m_2 - 1 \leq i \leq m_1 - 1$:
\[p_{i} - p_{i-1} \geq \frac{1/3}{m} - \frac{1}{50m^2} \geq \frac{1}{4m}, \] 
as required.
\end{proof}

\begin{proof}[Proof of \Cref{lem:AtoP}]
Let $i$ be the index guaranteed by \Cref{lem:reduction} such that $p_{i} - p_{i-1}\geq 1/4m$.
Pick an increasing realizable sample $S\in \bigl(X'\times\{\pm 1\}\bigr)^m$ so that the interval $J\subseteq X'$ between $x_{i-1}$ and $x_{i+1}$,
\[J = \bigl\{x\in \{1,\ldots, k\} : x_{i-1} < x < x_{i+1}\bigr\},\] 
is of size $k-m$.
For every $x\in J$ let $S_x$ be the neighboring sample of $S$ 
that is obtained by replacing $x$ with $x_i$.
This yields family of neighboring samples $\bigl\{S_x: x\in (x_{i-1},x_{i+1})\bigr\}$ such that
\begin{itemize}
\item every two output-distributions $A(S_{x'})$, $A(S_{x''})$ are $(\eps,\delta)$-indistinguishable (because $A$ satisfies $(\eps,\delta)$ differential privacy).
\item Set $r = \frac{p_{i+1} + p_i}{2}$. Then for all $x,x'\in J$: 
\[\Pr_{h\sim A(S_x) }\bigl[h(x')=1\bigr] =
\begin{cases}
\leq r-\frac{1}{10m} & x' < x,\\
\geq r+\frac{1}{10m} & x' >x.
\end{cases}\]
The proof is concluded by restricting the output of $A$ to $J$, 
and identifying $J$ with $[n]$ 
and each output-distributions $A(S_x)$ with a distribution over $\{\pm 1 \}^n$. 
\end{itemize}
\end{proof}

\subsubsection{Proof of \Cref{lem:binary}}

\begin{proof}

Set $T=10^3 m^2\log^2 m - 1$, and $D = 10^2m^2\log T$. 
We want to show that $n\leq 2^{T+1}$.
Assume towards contradiction that~$n > 2^{T+1}$.
Consider the family of distributions $Q_i=P_i^D$ for $i=1,\ldots,n$.
By Lemma~\ref{lem:prod}, each $Q_i,Q_j$ are $(0.1D,\delta D)$-indistinguishable.

We next define a set of mutually disjoint events $E_i$ for $i\leq 2^T$ that are measurable with respect to each of the~$Q_i$'s.
For a sequence  of vectors $\vv=(v_1,\ldots,v_D)$ in $\{\pm 1\}^n$ we let $\bar{\vv}\in \{\pm 1\}^n$ be the threshold vector defined by
\[\bar{\vv}(j) = \begin{cases}
-1 & \frac{1}{D}\sum_{i=1}^D  v_i(j) \le r, \\
+1 & \frac{1}{D}\sum_{i=1}^D  v_i(j) \ge r.
\end{cases} \]

Given a point in the support of any of the $Q_i$'s, 
namely a sequence $\vv = (v_1,\ldots, v_{D})$ of $D$ vectors in $\{\pm 1\}^n$ define a mapping $B$
according to the outcome of $T$ steps of binary search on $\bar{\vv}$ as follows: 
probe the $\frac{n}{2}$'th entry of $\bar{\vv}$; if it is $+1$ then continue recursively with the first half of $\bar{\vv}$. 
Else, continue recursively with the second half of $\bar{\vv}$. 
Define the mapping $B=B(\vv)$ to be the entry that was probed at the $T$'th step.
The events $E_j$ correspond to the $2^T$ different outcomes of $B$.
These events are mutually disjoint by the assumption that $n > 2^{T+1}$. 

Notice that for any possible $i$ in the image of $B$, 
applying the binary search on a sufficiently large i.i.d sample $\vv$ from $P_i$
would yield $B(\vv) = i$ with high probability.
Quantitatively, a standard application of Chernoff inequality and a union bound
imply that the event $E_i= \{\vv : B(\bar{\vv})=i\}$ for $\vv \sim Q_i$,  has probability at least 
\[1 - T\exp\Bigl(-2\frac{1}{10^2m^2}D\Bigr) = 1-T\exp(-2\log T) \geq \frac{2}{3}.\]
We claim that for all $j\leq n$, and $i$ in the image of $B$:
\begin{equation}\label{eq:crucial}
Q_j(E_i) \geq \frac{1}{2}\exp(-0.1D).
\end{equation}
This will finish the proof since the $2^T$ events are mutually disjoint, and therefore
\begin{align*}
1 &\geq Q_j(\cup_i E_i)\\
   &= \sum_i Q_j(E_i) \\
   &\geq 2^T \cdot \frac{1}{2}e^{-0.1 D}\\
& = 2^{T-1} e^{-0.1 D},
\end{align*}
however, $2^{T-1} e^{-0.1 D} > 1$ by the choice of $T,D$, which is a contradiction.

Thus it remains to prove \Cref{eq:crucial}. 
This follows since $Q_i,Q_j$ are $(0.1D,D\delta)$-indistinguishable:
\[\frac{2}{3} \leq Q_i(E_i) \leq \exp(0.1 D) Q_j(E_i) + D\delta,\]
and by the choice of $\delta$, which implies that $\frac{2}{3}-D\delta \geq \frac{1}{2}$.
\end{proof}

\section{Privately learnable classes have finite Littlestone dimension}\label{sec:lsdim}
We conclude the paper by deriving \Cref{thm:ADPimpliesLD} that gives a lower bound
of $\Omega(\log^* d)$ on the sample complexity of privately learning a class with Littlestone dimension $d$.
\begin{proof}[Proof of \Cref{thm:ADPimpliesLD}]
The proof is a direct corollary of \Cref{thm:shelah} and \Cref{thm:main}.
Indeed, let $H$ be a class with Littlestone dimension $d$, and let $c= \lfloor \log d\rfloor$.
By Item 1 of \Cref{thm:shelah}, there are $x_1,\ldots, x_c$ and $h_1,\ldots, h_c\in H$
such that $h_i(x_j) = +1$ if and only if $j\geq i$.
\Cref{thm:main} implies a lower bound of $m\geq \Omega(\log^* c) = \Omega(\log^* d)$
for any algorithm that learns $\{h_i : i\leq c\}$ with accuracy $(1/16,1/16)$
and privacy $(0.1,O(1/m^2\log m))$.
\end{proof}

\section{Conclusion}


The main result of this paper is a lower bound on the sample complexity of private learning
in terms of the Littlestone dimension.

We conclude with an open problem. 
There are many mathematically interesting classes with 
finite Littlestone dimension, see e.g.\ \citep{chase2018model}. It is natural to ask whether 
the converse to our main result holds, i.e.\ whether every class with finite Littlestone dimension may be 
learned privately.


\section*{Acknowledgements}
We thank Raef Bassily and Mark Bun for many insightful discussions
and suggestions regarding an earlier draft of this manuscript.
We also thank the anonymous reviewers 
for helping improving the presentation of this paper.

\bibliographystyle{plainnat}
\bibliography{ref}

\begin{thebibliography}{39}
\providecommand{\natexlab}[1]{#1}
\providecommand{\url}[1]{\texttt{#1}}
\expandafter\ifx\csname urlstyle\endcsname\relax
  \providecommand{\doi}[1]{doi: #1}\else
  \providecommand{\doi}{doi: \begingroup \urlstyle{rm}\Url}\fi

\bibitem[Balcan and Feldman(2015)]{Balcan15active}
Maria{-}Florina Balcan and Vitaly Feldman.
\newblock Statistical active learning algorithms for noise tolerance and
  differential privacy.
\newblock \emph{Algorithmica}, 72\penalty0 (1):\penalty0 282--315, 2015.

\bibitem[Bassily et~al.(2016)Bassily, Nissim, Smith, Steinke, Stemmer, and
  Ullman]{Bassily16stability}
Raef Bassily, Kobbi Nissim, Adam~D. Smith, Thomas Steinke, Uri Stemmer, and
  Jonathan Ullman.
\newblock Algorithmic stability for adaptive data analysis.
\newblock In \emph{{STOC}}, pages 1046--1059. {ACM}, 2016.

\bibitem[Bassily et~al.(2018)Bassily, Thakkar, and Thakurta]{Bassily18model}
Raef Bassily, Om~Thakkar, and Abhradeep Thakurta.
\newblock Model-agnostic private learning via stability.
\newblock \emph{CoRR}, abs/1803.05101, 2018.

\bibitem[Beimel et~al.(2013)Beimel, Nissim, and Stemmer]{Beimel13charac}
Amos Beimel, Kobbi Nissim, and Uri Stemmer.
\newblock Characterizing the sample complexity of private learners.
\newblock In \emph{{ITCS}}, pages 97--110. {ACM}, 2013.

\bibitem[Beimel et~al.(2014)Beimel, Brenner, Kasiviswanathan, and
  Nissim]{Beimel14bounds}
Amos Beimel, Hai Brenner, Shiva~Prasad Kasiviswanathan, and Kobbi Nissim.
\newblock Bounds on the sample complexity for private learning and private data
  release.
\newblock \emph{Machine Learning}, 94\penalty0 (3):\penalty0 401--437, 2014.

\bibitem[Beimel et~al.(2015)Beimel, Nissim, and Stemmer]{Beimel15unlabeled}
Amos Beimel, Kobbi Nissim, and Uri Stemmer.
\newblock Learning privately with labeled and unlabeled examples.
\newblock In \emph{{SODA}}, pages 461--477. {SIAM}, 2015.

\bibitem[Beimel et~al.(2016)Beimel, Nissim, and Stemmer]{Beimel16sanit}
Amos Beimel, Kobbi Nissim, and Uri Stemmer.
\newblock Private learning and sanitization: Pure vs. approximate differential
  privacy.
\newblock \emph{Theory of Computing}, 12\penalty0 (1):\penalty0 1--61, 2016.

\bibitem[Ben{-}David et~al.(2009)Ben{-}David, P{\'{a}}l, and
  Shalev{-}Shwartz]{Bendavid09agnostic}
Shai Ben{-}David, D{\'{a}}vid P{\'{a}}l, and Shai Shalev{-}Shwartz.
\newblock Agnostic online learning.
\newblock In \emph{{COLT}}, 2009.

\bibitem[Blum et~al.(2005)Blum, Dwork, McSherry, and Nissim]{Blum05practical}
Avrim Blum, Cynthia Dwork, Frank McSherry, and Kobbi Nissim.
\newblock Practical privacy: the sulq framework.
\newblock In \emph{{PODS}}, pages 128--138. {ACM}, 2005.

\bibitem[Blumer et~al.(1989)Blumer, Ehrenfeucht, Haussler, and
  Warmuth]{Blumer89learnability}
Anselm Blumer, Andrzej Ehrenfeucht, David Haussler, and Manfred~K. Warmuth.
\newblock {Learnability and the Vapnik-Chervonenkis dimension.}
\newblock \emph{{J. Assoc. Comput. Mach.}}, 36\penalty0 (4):\penalty0 929--965,
  1989.
\newblock ISSN 0004-5411.
\newblock \doi{10.1145/76359.76371}.

\bibitem[Bun et~al.(2015)Bun, Nissim, Stemmer, and Vadhan]{Bun15thresholds}
Mark Bun, Kobbi Nissim, Uri Stemmer, and Salil~P. Vadhan.
\newblock Differentially private release and learning of threshold functions.
\newblock In \emph{{FOCS}}, pages 634--649. {IEEE} Computer Society, 2015.

\bibitem[Bun et~al.(2016)Bun, Nissim, and Stemmer]{Bun16direct}
Mark Bun, Kobbi Nissim, and Uri Stemmer.
\newblock Simultaneous private learning of multiple concepts.
\newblock In \emph{{ITCS}}, pages 369--380. {ACM}, 2016.

\bibitem[Bun(2016)]{bun16thesis}
Mark~Mar Bun.
\newblock \emph{New Separations in the Complexity of Differential Privacy}.
\newblock PhD thesis, Harvard University, Graduate School of Arts \& Sciences,
  2016.

\bibitem[Chase and Freitag(2018)]{chase2018model}
Hunter Chase and James Freitag.
\newblock Model theory and machine learning.
\newblock \emph{arXiv preprint arXiv:1801.06566}, 2018.

\bibitem[Chaudhuri et~al.(2011)Chaudhuri, Monteleoni, and
  Sarwate]{Chaudhuri11erm}
Kamalika Chaudhuri, Claire Monteleoni, and Anand~D. Sarwate.
\newblock Differentially private empirical risk minimization.
\newblock \emph{Journal of Machine Learning Research}, 12:\penalty0 1069--1109,
  2011.

\bibitem[Chaudhuri et~al.(2014)Chaudhuri, Hsu, and Song]{Chaudhuri14margin}
Kamalika Chaudhuri, Daniel~J. Hsu, and Shuang Song.
\newblock The large margin mechanism for differentially private maximization.
\newblock In \emph{{NIPS}}, pages 1287--1295, 2014.

\bibitem[Cummings et~al.(2016)Cummings, Ligett, Nissim, Roth, and
  Wu]{Cummings16robust}
Rachel Cummings, Katrina Ligett, Kobbi Nissim, Aaron Roth, and Zhiwei~Steven
  Wu.
\newblock Adaptive learning with robust generalization guarantees.
\newblock In \emph{{COLT}}, volume~49 of \emph{{JMLR} Workshop and Conference
  Proceedings}, pages 772--814. JMLR.org, 2016.

\bibitem[Dwork and Feldman(2018)]{Feldman18prediction}
Cynthia Dwork and Vitaly Feldman.
\newblock Privacy-preserving prediction.
\newblock \emph{CoRR}, abs/1803.10266, 2018.

\bibitem[Dwork and Lei(2009)]{Dwork09robust}
Cynthia Dwork and Jing Lei.
\newblock Differential privacy and robust statistics.
\newblock In \emph{{STOC}}, pages 371--380. {ACM}, 2009.

\bibitem[Dwork and Roth(2014)]{Dwork14survey}
Cynthia Dwork and Aaron Roth.
\newblock The algorithmic foundations of differential privacy.
\newblock \emph{Foundations and Trends in Theoretical Computer Science},
  9\penalty0 (3-4):\penalty0 211--407, 2014.

\bibitem[Dwork et~al.(2006{\natexlab{a}})Dwork, Kenthapadi, McSherry, Mironov,
  and Naor]{Dwork06ourdata}
Cynthia Dwork, Krishnaram Kenthapadi, Frank McSherry, Ilya Mironov, and Moni
  Naor.
\newblock Our data, ourselves: Privacy via distributed noise generation.
\newblock In \emph{{EUROCRYPT}}, volume 4004 of \emph{Lecture Notes in Computer
  Science}, pages 486--503. Springer, 2006{\natexlab{a}}.

\bibitem[Dwork et~al.(2006{\natexlab{b}})Dwork, McSherry, Nissim, and
  Smith]{Dwork06calib}
Cynthia Dwork, Frank McSherry, Kobbi Nissim, and Adam~D. Smith.
\newblock Calibrating noise to sensitivity in private data analysis.
\newblock In \emph{{TCC}}, volume 3876 of \emph{Lecture Notes in Computer
  Science}, pages 265--284. Springer, 2006{\natexlab{b}}.

\bibitem[Erd\H{o}s and Rado(1952)]{erdos52combinatorial}
P.~Erd\H{o}s and R.~Rado.
\newblock Combinatorial theorems on classifications of subsets of a given set.
\newblock \emph{Proceedings of the London Mathematical Society}, s3-2\penalty0
  (1):\penalty0 417--439, 1952.
\newblock \doi{10.1112/plms/s3-2.1.417}.
\newblock URL
  \url{https://londmathsoc.onlinelibrary.wiley.com/doi/abs/10.1112/plms/s3-2.1.417}.

\bibitem[Feldman and Xiao(2015)]{Feldman15communication}
Vitaly Feldman and David Xiao.
\newblock Sample complexity bounds on differentially private learning via
  communication complexity.
\newblock \emph{{SIAM} J. Comput.}, 44\penalty0 (6):\penalty0 1740--1764, 2015.

\bibitem[Graham et~al.(1990)Graham, Graham, Rothschild, and
  Spencer]{graham90ramsey}
R.L. Graham, R.L. Graham, B.L. Rothschild, and J.H. Spencer.
\newblock \emph{Ramsey Theory}.
\newblock A Wiley-Interscience publication. Wiley, 1990.
\newblock ISBN 9780471500469.
\newblock URL \url{https://books.google.com/books?id=55oXT60dC54C}.

\bibitem[Hodges(1997)]{Hodges97book}
Wilfrid Hodges.
\newblock \emph{A Shorter Model Theory}.
\newblock Cambridge University Press, New York, NY, USA, 1997.
\newblock ISBN 0-521-58713-1.

\bibitem[Karpinski and Macintyre(1997)]{karpinski1997polynomial}
Marek Karpinski and Angus Macintyre.
\newblock Polynomial bounds for vc dimension of sigmoidal and general pfaffian
  neural networks.
\newblock \emph{Journal of Computer and System Sciences}, 54\penalty0
  (1):\penalty0 169--176, 1997.

\bibitem[Kasiviswanathan et~al.(2011)Kasiviswanathan, Lee, Nissim,
  Raskhodnikova, and Smith]{Kasiv11learning}
Shiva~Prasad Kasiviswanathan, Homin~K. Lee, Kobbi Nissim, Sofya Raskhodnikova,
  and Adam~D. Smith.
\newblock What can we learn privately?
\newblock \emph{{SIAM} J. Comput.}, 40\penalty0 (3):\penalty0 793--826, 2011.

\bibitem[Laskowski(1992)]{laskowski1992vapnik}
Michael~C Laskowski.
\newblock Vapnik-chervonenkis classes of definable sets.
\newblock \emph{Journal of the London Mathematical Society}, 2\penalty0
  (2):\penalty0 377--384, 1992.

\bibitem[Ligett et~al.(2017)Ligett, Neel, Roth, Waggoner, and
  Wu]{Ligett17accuracy}
Katrina Ligett, Seth Neel, Aaron Roth, Bo~Waggoner, and Steven~Z. Wu.
\newblock Accuracy first: Selecting a differential privacy level for accuracy
  constrained {ERM}.
\newblock In \emph{{NIPS}}, pages 2563--2573, 2017.

\bibitem[Littlestone(1987)]{Littlestone87online}
Nick Littlestone.
\newblock Learning quickly when irrelevant attributes abound: {A} new
  linear-threshold algorithm.
\newblock \emph{Machine Learning}, 2\penalty0 (4):\penalty0 285--318, 1987.

\bibitem[Livni and Simon(2013)]{livni2013honest}
Roi Livni and Pierre Simon.
\newblock Honest compressions and their application to compression schemes.
\newblock In \emph{Conference on Learning Theory}, pages 77--92, 2013.

\bibitem[{Mubayi} and {Suk}(2017)]{mubayi17survey}
D.~{Mubayi} and A.~{Suk}.
\newblock {A survey of quantitative bounds for hypergraph Ramsey problems}.
\newblock \emph{ArXiv e-prints}, July 2017.

\bibitem[Rubinstein et~al.(2009)Rubinstein, Bartlett, Huang, and
  Taft]{Rubinstein09large}
Benjamin I.~P. Rubinstein, Peter~L. Bartlett, Ling Huang, and Nina Taft.
\newblock Learning in a large function space: Privacy-preserving mechanisms for
  {SVM} learning.
\newblock \emph{CoRR}, abs/0911.5708, 2009.

\bibitem[Shalev-Shwartz and Ben-David(2014)]{Shalev14book}
Shai Shalev-Shwartz and Shai Ben-David.
\newblock \emph{Understanding Machine Learning: From Theory to Algorithms}.
\newblock Cambridge University Press, New York, NY, USA, 2014.
\newblock ISBN 1107057132, 9781107057135.

\bibitem[Shelah(1978)]{Shelah78classification}
Saharon. Shelah.
\newblock \emph{Classification theory and the number of non-isomorphic models}.
\newblock North-Holland Pub. Co. ; sole distributors for the U.S.A. and Canada,
  Elsevier/North-Holland Amsterdam ; New York : New York, 1978.
\newblock ISBN 0720407575.

\bibitem[Vadhan(2017)]{Vadhan17survey}
Salil~P. Vadhan.
\newblock The complexity of differential privacy.
\newblock In \emph{Tutorials on the Foundations of Cryptography}, pages
  347--450. Springer International Publishing, 2017.

\bibitem[{Vapnik} and {Chervonenkis}(1971)]{Vapnik71uniform}
V.N. {Vapnik} and A.Ya. {Chervonenkis}.
\newblock {On the uniform convergence of relative frequencies of events to
  their probabilities.}
\newblock \emph{{Theory Probab. Appl.}}, 16:\penalty0 264--280, 1971.
\newblock ISSN 0040-585X; 1095-7219/e.
\newblock \doi{10.1137/1116025}.

\bibitem[Wang et~al.(2016)Wang, Lei, and Fienberg]{Wang16learning}
Yu{-}Xiang Wang, Jing Lei, and Stephen~E. Fienberg.
\newblock Learning with differential privacy: Stability, learnability and the
  sufficiency and necessity of {ERM} principle.
\newblock \emph{Journal of Machine Learning Research}, 17:\penalty0
  183:1--183:40, 2016.
\newblock URL \url{http://jmlr.org/papers/v17/15-313.html}.

\end{thebibliography}

\appendix

\section{Proof of Lemma~\ref{lem:prod}}\label{app:prod}

The theorem  follows by induction from the following lemma.

\begin{lemma}
Let $p_1,q_1$ be distributions over a countable domain $X_1$  
and $p_2,q_2$ be distributions over a countable domain $X_2$.
Assume that $p_i,q_i$ are $(\eps_i,\delta_i)$-indistinguishable for $i=1,2$.
Then $p_1\times p_2, q_1\times q_2$ are $(\eps_1+\eps_2,\delta_1+\delta_2)$-indistinguishable.
\end{lemma}
\begin{proof}
Let $a\land b$ denote $\min\{a,b\}$.
For $x\in X_i$ let 
\[\Delta_i(x) = 
\begin{cases}
p_i(x) - e^\eps q_i(x) &p_i(x) - e^\eps q_i(x)\geq0,\\
0	&p_i(x) - e^\eps q_i(x) < 0.
\end{cases}\]
Extend $\Delta_1,\Delta_2$ to be a measure on $X$ in the obvious way.
Note that
\begin{itemize}
\item $\Delta_i(X) \leq \delta_i$, and that
\item $p_i(x) \leq e^\eps q_i(x) + \Delta_i(x)$  for all $x\in X_i$.
\end{itemize}

Let $S\subseteq X_1\times X_2$. 
We show that $(p_1\times p_2)\bigl(S\bigr)\leq e^{\eps_1+\eps_2}\bigl(q_1\times q_2\bigr)(S) + \delta_1 + \delta_2$,
the other direction can be derived similarly.

For $a\in X_1$ let $S_a\subseteq X_2$ denote the set $\{b : (a,b)\in S\}$
\begin{align*}
(p_1\times p_2)\bigl(S\bigr) &= \sum_{a\in X_1}p_1(a)p_2(S_a)\\
				           &\leq \sum_{a\in X_1}p_1(a)\Bigl(\bigl(1\land e^{\eps_2}q_2(S_a)\bigr) + \delta_2\Bigr)\\
				           &\leq \delta_2 + \sum_{a\in X_1}p_1(a)\bigl(1\land e^{\eps_2}q_2(S_a)\bigr)\\
				           &\leq \delta_2 + \sum_{a\in X_1}\bigl(e^{\eps_1}q_1(a) + \Delta_1(a)\bigr)\bigl(1\land e^{\eps_2}q_2(S_a)\bigr)\\ 
				           &\leq \delta_2  + \sum_{a\in X_1}e^{\eps_1}q_1(a)e^{\eps_2}q_2(S_a) + \Delta_1(a)\\
				           &= \delta_2 + \Delta_1(X) + e^{\eps_1+\eps_2}\bigl(q_1\times q_2\bigr)(S)\\
				           &\leq e^{\eps_1+\eps_2}\bigl(q_1\times q_2\bigr)(S) + \delta_1 + \delta_2.
\end{align*}
\end{proof}

\section{Proof of \Cref{thm:shelah}}\label{sec:shelah}

In this appendix we prove \Cref{thm:shelah}.  Throughout the proof
a labeled binary tree means a full binary tree whose internal vertices are
labeled by instances. 

The second part of the theorem
is easy. If
$\cH$ contains $2^t$ thresholds then there are $h_i \in \cH$ 
for $0 \leq i < 2^t$ and there are $x_j$ for $0 \leq j <2^t-1$
such that $h_i(x_j)=0$  for $j<i$ and $h_i(x_j)=1$ for $j \geq i$.
Define a labeled binary tree of height $t$
corresponding to the binary search process.
That is, the root is labeled  by $x_{2^{t-1}-1}$, its left
child by $x_{2^{t-1}+2^{t-2}-1}$ and its right child by 
$x_{2^{t-1}-2^{t-2}-1}$
and so on. If the label of an internal vertex of distance $q$ from
the root, where $0 \leq q \leq t-1$, 
is $x_p $, then the label of its left child is
$x_{p+2^{t-q-1}}$ and the label of its right child is
$x_{p-2^{t-q-1}}$. It is easy to check that the root-to-leaf path
corresponding to each 
of the functions $h_i$ leads to leaf number $i$ from the right
among the leaves of the tree (counting from $0$ to $2^t-1$). 

To prove the first part of the theorem 
we first define the notion of a subtree $T'$ of
depth $h$ of a labeled binary tree $T$ by induction on
$h$. Any leaf of $T$ is a subtree of height $0$. For $h \geq 1$
a subtree of height $h$ is obtained from an internal vertex
of $T$ together with a subtree of height $h-1$ of the tree rooted at
its left child and a subtree of height $h-1$ of the tree rooted at its
right child. Note that if $T$ is a labeled tree
and  it is shattered by the class $\cH$, then any
subtree $T'$ of it  with the same labeling of its internal vertices
is shattered by the class $\cH$.  With this definition we prove the following
simple lemma.

\begin{lemma}
\label{ltree}
Let $p,q $ be positive integers and let
$T$ be a labeled binary tree of height $p+q-1$ whose internal vertices
are colored by two colors, red and blue. Then $T$ contains either
a subtree of height $p$ in which all internal vertices are red
(a red subtree),
or a subtree of height $q$ in which all vertices are blue
(a blue subtree).
\end{lemma}
\vspace{0.1cm}

\noindent
{\bf Proof:}\, 
We apply induction on $p+q$. The result is trivial for $p=q=1$
as the root of $T$ is either  red or blue.  Assuming the assertion holds
for $p'+q'<p+q$, let $T$ be of height $p+q-1$. Without loss of generality
assume the root of $t$ is red. If $p=1$ we are done, as the root
together with a leaf in the subtree of its left child and one in
the subtree of its right child
form a red subtree of height $p$. If $p>1$ then, by the induction
hypothesis, the tree rooted at the left child of the root of $T$
contains either a red subtree of height $p-1$ or a blue subtree
of height $q$, and the same applies to the tree rooted at the right 
child of the root. If at least one of them contains a blue subtree as above
we are done, otherwise, the two red subtrees together with the root 
provide the required red subtree. $\Box$
\vspace{0.2cm}

\noindent
We can now prove the first part of the theorem, showing that if the 
Littlestone dimension of $\cH$ is at least $2^{t+1}-1$ then
$\cH$ contains $t+2$ thresholds. We apply induction on $t$.
If $t=0$ we have a tree of height $1$ shattered by $\cH$. 
Its root is labeled by some variable $x_0$ and as it is shattered
there are two functions $h_0,h_1 \in \cH$ so that 
$h_0(x_0)=1, h_1(x_0)=0$, meaning that $\cH$ contains two thresholds,
as needed.
Assuming the desired result holds for $t-1$ we prove it for
$t$, $t \geq 1$. Let $T$ be a labeled binary tree of height
$2^{t+1}-1$ shattered by $\cH$. Let $h $ be an arbitrary 
member of $\cH$ and define a two coloring of the internal vertices 
of $T$ as follows. If an internal vertex is labeled by $x$ and
$h(x)=1$ color it red, else color it blue. Since
$2^{t+1}-1=2 \cdot 2^t-1$, Lemma \ref{ltree}
with  $p=q=2^t$ implies that $T$ contains either a red or a blue 
subtree $T'$ of height $2^t$. In the first case define $h_0=h$ and
let $X$ be the set of all variables $x$ so that $h(x)=1$. 
Let $x_0$ be the root of $T'$ and let  $T''$ be the subtree of $T'$
rooted at the left child of $T'$. Let $\cH'$ be the set of all
$h' \in \cH$ so that $h'(x_0)=0$. Note that $\cH'$ shatters the 
tree $T''$, and that the depth of $T''$ is $2^t-1$. We can thus
apply the induction hypothesis and get a set of $t+1$ thresholds
$h_1,h_2, \ldots ,h_{t+1} \in \cH'$ and variables 
$x_1,x_2, \ldots ,x_t \in X$
so that $h_i(x_j)=1$ iff $j \geq i$. Adding $h_0$ and $x_0$ to these
we get the desired $t+2$ thresholds.

Similarly, if $T$ contains a blue subtree $T'$, define
$h_{t+1}=h$ and let $X$ be the set of all variables $x$ so that
$h(x)=0$. In this case  denote the root of $T'$ by $x_{t}$
and let $T''$ be the subtree of $T'$ rooted at the right child
of $T'$. Let $\cH'$ be the set of all $h' \in  \cH$ so that
$h'(x_{t})=1$. As before, $\cH'$ shatters the tree $T''$ whose
depth is $2^t-1$. By the induction
hypothesis we get $t+1$ thresholds $h_0,h_1, \ldots ,h_t$ 
and variables $x_0, x_1, \ldots ,x_{t-1} \in X$ so that
$h_i(x_j)=1$ iff $j \geq i$, and the desired result follows by
appending to them $h_{t+1}$ and $x_t$. This completes the proof.
$\Box$

\end{document}